\documentclass[lettersize,journal]{IEEEtran}
\usepackage{amsmath,amsfonts}
\usepackage{algorithmic}
\usepackage{array}
\usepackage[caption=false,font=normalsize,labelfont=sf,textfont=sf]{subfig}
\usepackage{textcomp}
\usepackage{stfloats}
\usepackage{url}
\usepackage{verbatim}
\usepackage{graphicx}
\usepackage{algorithm,algorithmic}
\usepackage{amsmath,amsfonts,amssymb,amsthm,version}
\usepackage[T1]{fontenc} 
\usepackage{color}
\usepackage{multirow}
\newtheorem{definition}{Definition}
\newtheorem{theorem}{Theorem}
\newtheorem{lemma}{Lemma}

\def\BibTeX{{\rm B\kern-.05em{\sc i\kern-.025em b}\kern-.08em
    T\kern-.1667em\lower.7ex\hbox{E}\kern-.125emX}}
\usepackage{balance}
\begin{document}
\title{Granular-Ball Fuzzy Set and Its Implementation in SVM}
\author{Shuyin Xia, Xiaoyu Lian, Guoyin Wang*, Xinbo Gao, Yabin Shao
\thanks{S. Xia, X. Lian, G. Wang, X. Gao and Y. Shao are with the Chongqing Key Laboratory of Computational Intelligence, Chongqing University of Telecommunications and Posts, 400065, Chongqing, China. E-mail: xiasy@cqupt.edu.cn, 1258852995@qq.com, shaoyb@cqupt.edu.cn.}}


\maketitle

\begin{abstract}
Most existing fuzzy set methods use points as their input, which is the finest granularity from the perspective of granular computing. Consequently, these methods are neither efficient nor robust to label noise. Therefore, we propose a frame-work called granular-ball fuzzy set by introducing granular-ball computing into fuzzy set. The computational framework is based on the granular-balls input rather than points; therefore, it is more efficient and robust than traditional fuzzy methods, and can be used in various fields of fuzzy data processing according to its extensibility. Furthermore, the framework is extended to the classifier fuzzy support vector machine (FSVM), to derive the granular ball fuzzy SVM (GBFSVM). The experimental results demonstrate the effectiveness and efficiency of GBFSVM. The source codes and data sets are available on the public link: http://www.cquptshuyinxia.com/GBFSVM.html.

\end{abstract}


\section{Introduction}
\IEEEPARstart{I}{n} the practical world, there are numerous fuzzy phenomena or concepts in the objective world, such as big and small, light and heavy, fast and slow, dynamic and static, deep and shallow, beauty and ugliness, etc., which cannot be clearly and completely distinguished. In fact, fuzzy information is also reliable information. In order to quantitatively describe the objective laws of fuzzy concepts and fuzzy phenomena, Professor L.A. Zadeh, an American computer and cybernetics expert, put forward the important concept of fuzzy set \cite{lin2002fuzzy} in 1965. He used membership functions to represent fuzzy sets, which are functions of [0,1] closed intervals, to describe the degree to which elements belong to fuzzy sets. The greater the function value, the greater the degree of membership. Since Zadeh introduced fuzzy sets \cite{F3}, it has been applied to various fields such as control systems, pattern recognition, machine learning, etc, and its another branch, fuzzy rough set, has also been developed rapidly. Several scholars have conducted in-depth research in the direction of feature selection \cite{hu2011feature,lin2017streaming,sun2020feature,tan2018intuitionistic,wang2016fitting,wang2021feature}, clustering \cite{ding2021unsupervised}, decision making \cite{selvachandran2019new, zhan2021novel} , classification \cite{hu2006fuzzy} and so on.

\begin{figure}[!ht]
	\centering
	{\includegraphics[width = 0.25\textwidth]{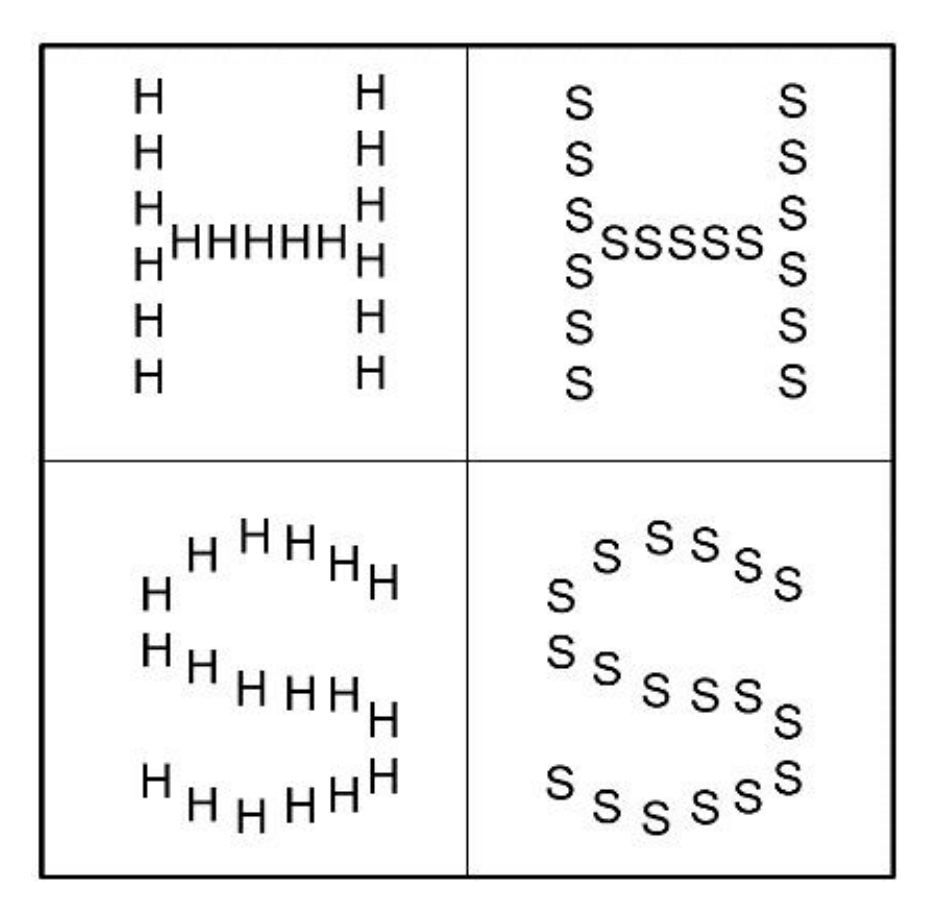}}
	\caption{Human cognition the coarse-grained large range is preferred.}
	\label{Humancognition}
\end{figure}

Considering the classification problem of fuzzy data sets, Lin et al. \cite{lin2002fuzzy} proposed a fuzzy support vector machine (FSVM) model by applying fuzzy membership to each input point. The model can make full use of the sample information, however, the complexity of the training stage is still high for a large number of data classification problems. For the research on fuzzy set classification tasks in the field of machine learning, Aydogan et al. \cite{aydogan2012hga} proposed a hybrid heuristic method based on the genetic algorithm (GA) and integer programming formula (IPF) to solve the high-dimensional classification problem in the classification system of linguistic fuzzy rules. The method can find accurate and concise classification rules, but can not flexibly consider the number of rule sets generated in the classification. Sanz et al. \cite{sanz2021wrapper} directly learned interval-valued fuzzy rules by defining a packaging method to obtain a classification system based on the interval valued fuzzy principle. Compared with the existing algorithm at that time, the accuracy of this method has been significantly improved, but the unbalanced classification problem can not be well tested. The algorithm is inefficient owing to its two evolutionary processes. Li et al. \cite{F4} proposed an interval extreme learning machine for interval fuzzy set classification of continuous-valued attributes, in which the discretization of conditional attributes and fuzzification of class labels are considered. Recently, an associative fuzzy classifier called CFM-BD \cite{F7} was been developed, which has shown robust predictive performance against more complex algorithms such as fuzzy decision trees \cite{F8}. To simplify the rule set, Aghaeipoor et al. \cite{F9} proposed a new scalable fuzzy classifier for big data, namely Chi-BD-DRF, which added the method of "dynamic rule filtering (DRF)" to supplement fuzzy big data learning.

\begin{figure}[!ht]
	\centering
	{\includegraphics[width = 0.50\textwidth]{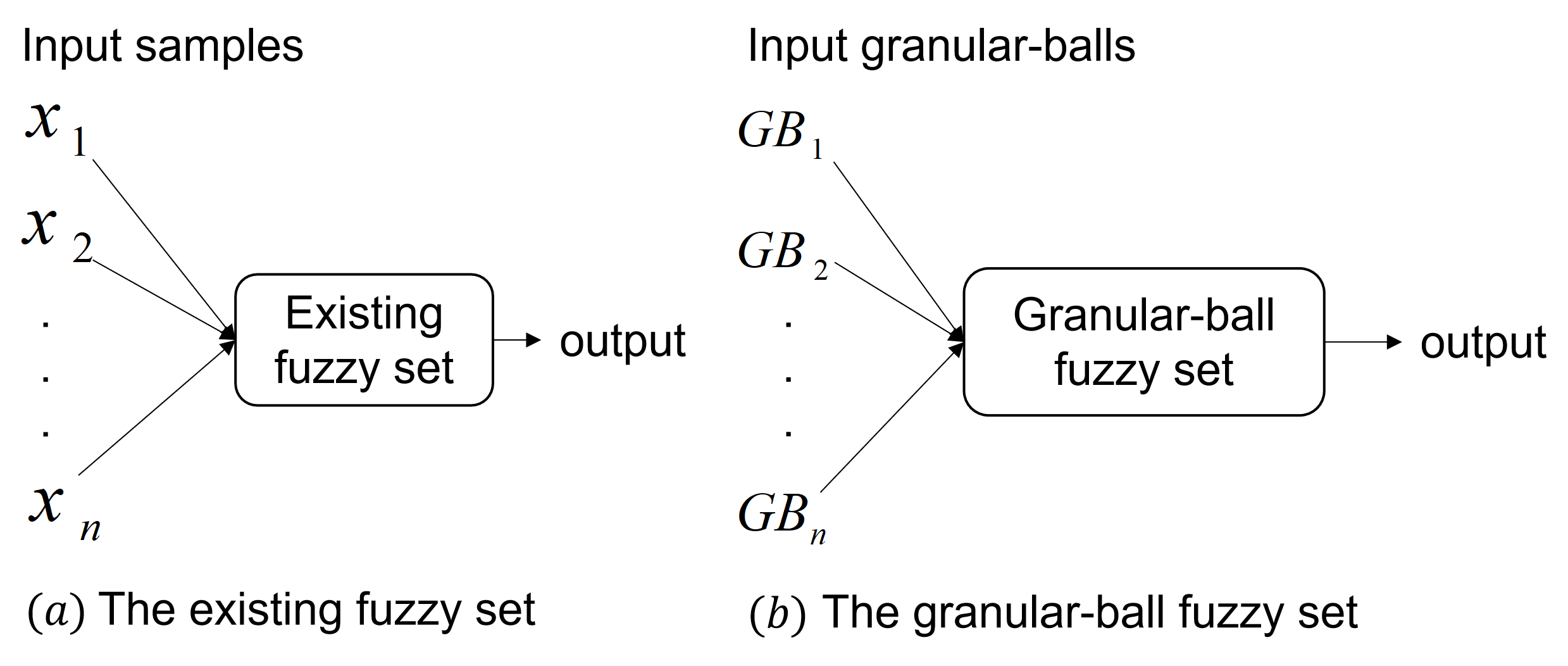}}
	\caption{The existing 
		fuzzy set and the comparison of the granular-ball fuzzy set.}
	\label{input}
\end{figure}

The aforementioned mentioned processing methods are based on the finest granularity from the perspective of granular computing \cite{G1,G2}, as shown in Fig. \ref{input}(a), therefore, it is not efficient and robust. Human cognition has the rule of "large scope first," and the visual system is particularly sensitive to the global topological characteristics, from large to small, from coarse-grained to fine-grained as shown in Fig. \ref{Humancognition} \cite{xia2022efficient}. In granular computing, the larger the granularity size, the higher the efficiency and the better the robustness to noise. However, this is also more likely to lead to a lack of detail and loss of accuracy. Smaller granularity allows more attention to detail, but may reduce the efficiency and robustness to label noise. In the past decades, scholars worldwide have been constantly studying \cite{G3,G5,G6,G7}, who granulate huge amounts of data and knowledge into different granularities according to different tasks. The relationship between these granularities was then used to solve this problem \cite{G8,G10,G11,yang2020granular}. Selecting different granularities according to various scenarios can improve the performance of multi-granularity learning methods and solve practical problems \cite{G14,G16,G17}. Therefore, Xia et al.\cite{G15} proposed granular-ball classifiers using some hyper-balls to granulate the dataset into different sizes of granular-balls \cite{G12}. The granular-ball support vector machine (GBSVM) \cite{xia2022gbsvm} is further proposed, and exhibits higher accuracy and efficiency than the traditional SVM. In order to improve the efficiency of fuzzy data processing, the idea of granular-ball computing can be introduced into fuzzy data processing by defining the fuzzy granular-ball, as shown in \ref{input}(b). The concept of fuzzy granular-balls was briefly proposed in our previous work \cite{shao}, but its algorithm is not designed; besides, its SVM model is incorrect \cite{shao,xue2021dual}, too complex and not consistent with the SVM. In order to improve the efficiency and robustness of fuzzy classifiers by combining granular-ball computing, the main contributions of the paper are as follows:


\begin{itemize}
	\item[$\bullet$] We propose a framework called the granular-ball fuzzy set by introducing the concept of the fuzzy granular-ball. It is different from the traditional fuzzy data processing method. 
	\item[$\bullet$] GBFSVM is proposed based on the fuzzy granular-ball framework. The framework uses granular-balls as the basic analysis unit instead of data points.  
	
	\item[$\bullet$] Considering the classification problem with the characteristics of triangular fuzzy numbers, the GBFSVM based on triangular fuzzy numbers is derived in detail using the possibility measure theory.
	
	\item[$\bullet$] Particle swarm optimization (PSO) is used to solve the dual model of GBSVM. Experimental results indicate that GBFSVM performs better than the traditional SVM and FSVM both in robustness and effectiveness.

\end{itemize}

The rest of this paper is organized as follows: we introduce the concepts of fuzzy sets and the work related to granular-ball computing \ref{sec:related work}. Section \ref{sec3} details the granular-ball fuzzy set framework and the definition of fuzzy granular-ball. Section \ref{sec4} introduces the application of granular-ball fuzzy set in fuzzy support vector machines and support vector machines based on triangular fuzzy numbers. The experimental results and analysis are presented in Section \ref{sec:experiment}. Finally, some concluding remarks are given in Section \ref{sec6}.

\section{Related Work}\label{sec:related work}
\subsection{Related concepts of fuzzy sets}
\noindent With the development of modern science and technology, the system we are facing is becoming more and more complex. For complex problems in the fields of humanities, social sciences and other "soft sciences," it is often difficult to provide an accurate evaluation owing to insufficient cognition or information content in the decision-making process. For multi-attribute decision making without specific decision information, it is difficult for decision makers to accurately evaluate the scheme, thus, the concept of the fuzzy set is generated. This concept is as follows:
\begin{definition}\rm{\textbf{(Fuzzy set \cite{F10})}}\label{def:Fuzzy set}
	\setlength{\parindent}{2em}If $X$ is a collection of objects denoted generically by $x$, then a fuzzy set $\tilde{A}$ in $X$ is a set of ordered pairs:
	\begin{equation}\label{de2.1}
	\tilde{A}= \left\{x, \mu_{\tilde{A}} (x) | x \in X \right\}. 
	\end{equation}
	\setlength{\parindent}{2em}$\mu_{\tilde{A}}(x)$ is called the membership function (generalized characteristic function) which maps $X$ to the membership space $M$. Generally speaking, the range of the membership function is $[0,1]$.
\end{definition}

The most important role of fuzzy sets is to represent various uncertainties in the data and data processing. In particular, the introduction of fuzzy sets in big data improves the representation ability of the information samples.

In particular, the triangular fuzzy number is the concept of the fuzzy set proposed by Professor Lotfi A. Zadeh in 1965 in order to solve these problems in an uncertain environment. The concept of triangular fuzzy number is as follows:
\begin{definition}\rm{\textbf{(Triangular fuzzy number \cite{F12})}}\label{def:Triangular fuzzy number}
	Suppose $\tilde{a}$ is a triangular fuzzy number, when its membership function is expressed as follows:
	\begin{equation}\label{de2.2}
	\mu_{\tilde{a}}(x)=
	\left\{\begin{aligned}
	&\frac{x - r_1}{r_2 - r_1},\quad\quad r_1 \leq x\textless r_2,\\
	&1,\quad\quad \quad\quad\quad x = a ,\\
	&\frac{x-r_3}{r_2-r_3},\quad \quad r_2 \textless x \leq r_3,\\
	\end{aligned}\right.
	\end{equation}
	where $r_1 \leq r_2 \leq r_3,r_j \in \mathbb{R}(j=1,2,3)$ and $\tilde{a}$ is called a triangular fuzzy number, denoted by $\tilde{a}=(r_1,r_2,r_3)$. The real numbers $r_2$,  $r_1$ and $r_3$ are called the center, left and right endpoints of the triangular fuzzy number $\tilde{a}$, respectively.
	The center reflects the main position of the triangular fuzzy number, and the real number $a$ can be expressed as a special triangular fuzzy number $a=(a,a,a)$.
\end{definition}
The probability of occurrence of a fuzzy event ${\tilde{a}}$ can be measured using the possibility measure. This possibility measure was proposed by Professor Lotfi A. Zadeh in 1978. It is defined as follows:
\begin{definition}\rm{\textbf{(Possibility measure \cite{F11})}}\label{def:Possibility measure}
	\setlength{\parindent}{2em}Let ($\Gamma$, $\mathcal{A}$) be a backup domain space, and $\rm{Pos}$ is a set function defined on the backup domain $\mathcal{A}$. If $\rm{Pos}$ satisfies the following conditions:\\
	(1) ${\rm{Pos}}(\emptyset)=0$, and ${\rm{Pos}}(\Gamma)=1$;\\
	(2) For any subclass $\left\{A_i | i \in I\right\}$ of $\mathcal{A}$, there are ${\rm{Pos}}\left(\underset {i \in I} \bigcup A_i\right)= \underset {i \in I} \sup \ {\rm{Pos}}(A_i)$. This is called a possibility measure, and the triples ($\Gamma ,  \mathcal{A}, {\rm{Pos}}$) are called possibility spaces.
\end{definition}

When the triangular fuzzy number is used to represent the fuzzy event $a$, the likelihood of the fuzzy event is measured as follows:
\begin{definition} \rm{\textbf{\cite{F14}}} 
	\setlength{\parindent}{2em} Let ${\tilde{a}}=(r_1,r_2,r_3)$ be a triangular fuzzy number, then:
	
	\begin{equation}\label{de2.3}
	{\rm{Pos}}({\tilde{a}}\leq 0)=
	\left\{\begin{aligned}
	&1,\quad\quad \quad\quad\quad r_2 \leq 0, \\
	&\frac{r_1}{r_1 - r_2},\quad\quad r_1 \leq 0, r_2 \textgreater 0,\\
	&0,\quad\quad \quad\quad\quad r_1 \textgreater 0. \\
	\end{aligned}\right.
	\end{equation}
\end{definition}

\begin{lemma} \rm{\textbf{\cite{F13}}} \label{de2.4}
	\setlength{\parindent}{2em} Let ${\tilde{a}}=(r_1,r_2,r_3)$ be a triangular fuzzy number, then for any given confidence level $\lambda\left(0 \textless \lambda \leq 1\right)$, we have:\\
	\begin{equation} 
	{\rm{Pos}}\left\{{\tilde{a}}\leq 0\right\}\geq\lambda \Leftrightarrow (1-\lambda)r_1+\lambda r_2 \leq 0.
	\end{equation}
\end{lemma}

SVM is a powerful tool for solving classification problems, however, this theory still has some limitations. All training points of the same class are uniformly treated using SVM theory. In various real-world applications, the effects of training points are different, and all have an ambiguous membership relationship. Specifically, each training point no longer belongs to one of these two classes entirely. Whereas the parameter $\xi$ is a measure of the error in SVM, and FSVM considers adding different weights(i.e., membership degree ${\delta}_i$) to the error \cite{lin2002fuzzy}, and its model is as follows:
\begin{equation}\label{}
	\left\{\begin{aligned}
		&{\rm {min}} \ \ \ \frac{1}{2}\Vert w \Vert^2 + C \sum_{i=1}^l {\delta}_i \xi_i, \\
		&s.t. \quad y_i(wx_i+b)  \textgreater 1 - \xi_i , \\
		&\quad\quad\ \  \xi_i \geq 0,i=1,2,...,n.
	\end{aligned}\right.
\end{equation}

\subsection{Granular-ball Computing}

Granular-ball computing is a big data processing method proposed by Wang and Xia to meet the scalability of high-dimensional data \cite{G15}. The core idea of granular-ball computing involves using the hyper ball to cover all or part of the sample space, and use "granular-ball" as the input to represent the sample space, so as to achieve multi-granularity learning characteristics and accurate characterization of the sample space. The great advantage of this method is that it only needs two data representations of center and radius in any dimension.

In any dimensional space $R^d$, each granular-ball can be described by two parameters i.e. center $c$ and radius $r$. The detailed definition is as follows:

\begin{definition}\rm{\textbf{\cite{G15}}}
	Given a data set $D= \{x_1,x_2,...,x_n \}  \in \mathbb{R}^d$, the center $c$ of a granular-ball is the center of gravity for all sample points in the ball, and $r$ is equal to the average distance from all points in the granular-ball to $c$. Specifically, we have:
	$$
	c = \frac{1}{n} \sum_{i = 1}^{n} x_i; \ \ \ r = \frac{1}{n} \sum_{i = 1}^{n}  \left| x_i - c  \right|.
	$$
\end{definition}

The radius $r$ is defined as the average distance rather than the maximum distance. The balls generated with the average distance are not easily affected by the outlier sample and better fit the data distribution. The label of a granular-ball is defined as the label with the most appearances in a granular-ball. To quantitatively analyze the mass of the split granular-ball, the concept of "purity threshold" is proposed, and it is defined as the percentage of majority samples with the same label in a granular-ball.

\begin{figure}[!t]
	\centering
	\includegraphics[width=3.6in]{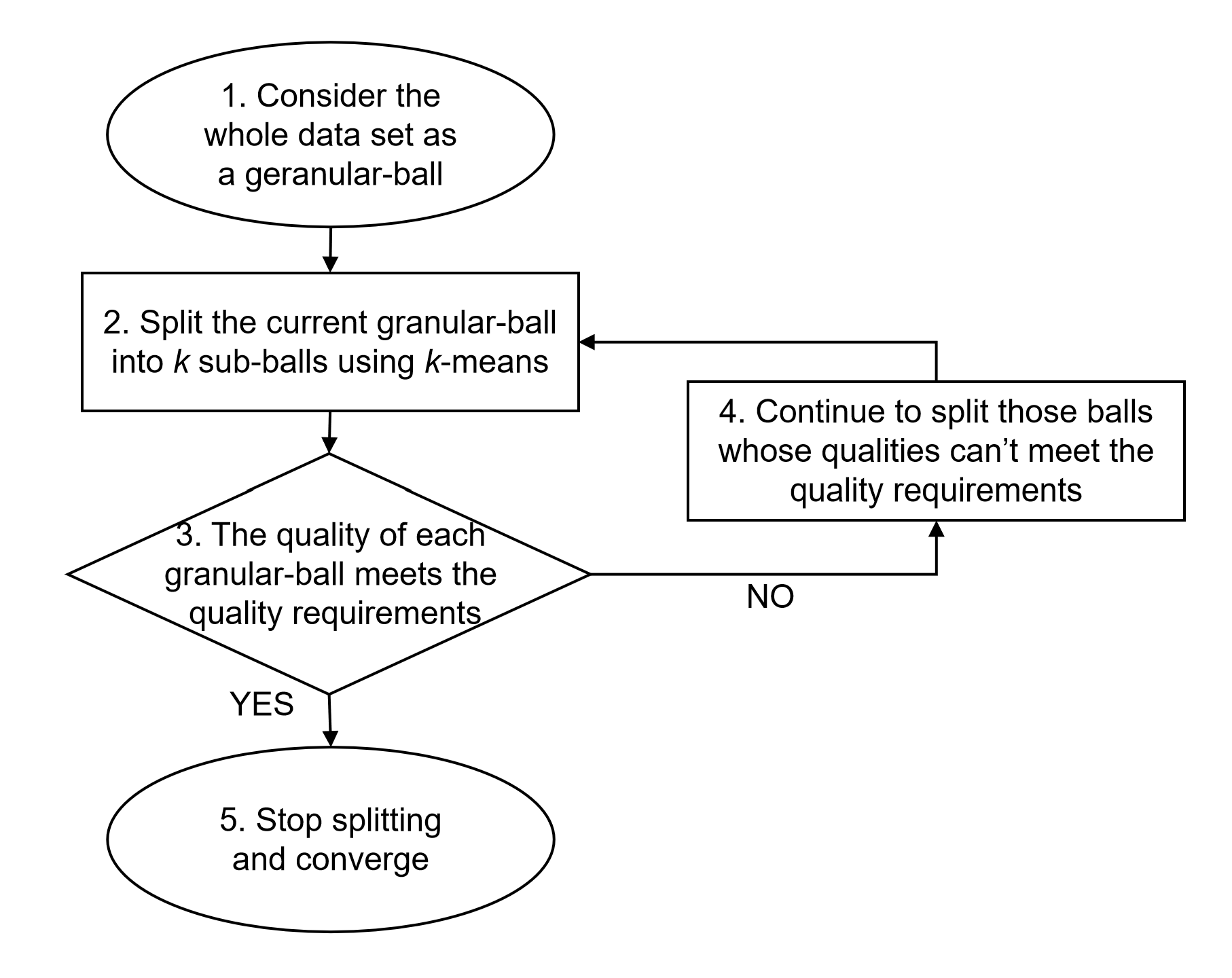}
	\caption{Process of the existing granular-ball generation in granular-ball computing.}
	\label{figure:GB}
\end{figure} 

Given the training set $D=\{x_i,i=1,2,...,n\}$, taking the reciprocal of the granular-ball covering is taken to optimize its minimum value. The optimization goal of the granular-balls can be expressed as:
\begin{equation}\label{GBmodel}
	\left\{\begin{aligned}
		&{\mathop{\min}  } \ \ {\lambda_1} * \frac{n}{\sum_{j = 1}^{m}  \left| GB_{j} \right|} + {\lambda_2} * m, \\
		&s.t. \quad quality (GB_{j}) \geq T, j=1,2,...,m, \\
	\end{aligned}\right.
\end{equation}
where $\lambda_1$ and $\lambda_2$ are the corresponding weight coefficients and $T$ is the purity threshold. $m$ represents the number of granular-balls, $GB$, and $m < n$ \cite{xia2022efficient}. The existing granular-ball splitting method currently uses the efficient $k$-means method ($k$ is the number of labels in a certain ball) to ensure the efficiency of the granular-ball classification process. Fig. \ref{figure:GB} is a heuristic algorithm to solve the model (\ref{GBmodel}). The dataset as a whole can be considered as a granular-ball at the beginning, as shown in Fig. \ref{generateGB}(a). At this instant, the quality of the granular-ball is the worst, and it cannot describe the distribution characteristics of the data. Therefore the granular-ball needs to be further split, and for each split, it was necessary to count the number of labels for the different categories in the granular-balls. The ball will continue to be divided and the purity of the granular-ball increases, as shown in Fig. \ref{generateGB} (b)-(d). When the purity of all the granular-balls meets the requirement, the algorithm converges result as shown in Fig. \ref{generateGB} (e). The resulting granular-balls are shown in Fig. \ref{generateGB} (f), which shows that granular-ball computing can well describe the distribution of data.

\begin{figure}[!t]
	\centering
	\includegraphics[width=1.7in]{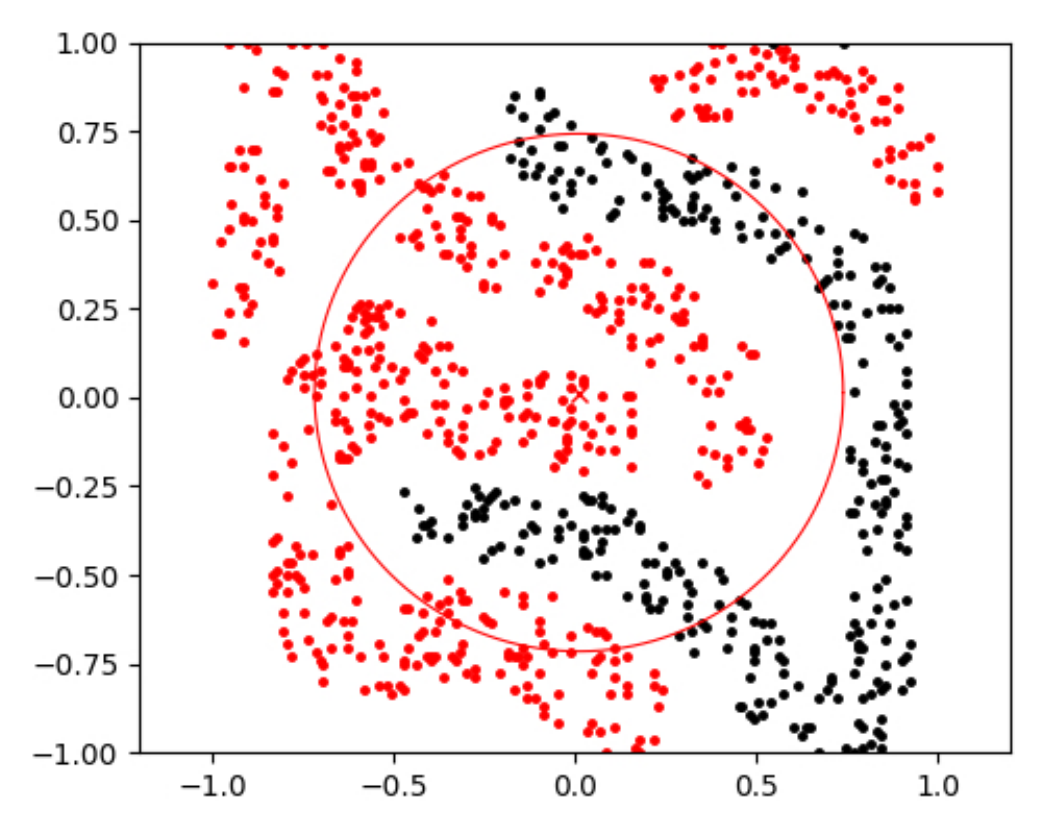}
	\includegraphics[width=1.7in]{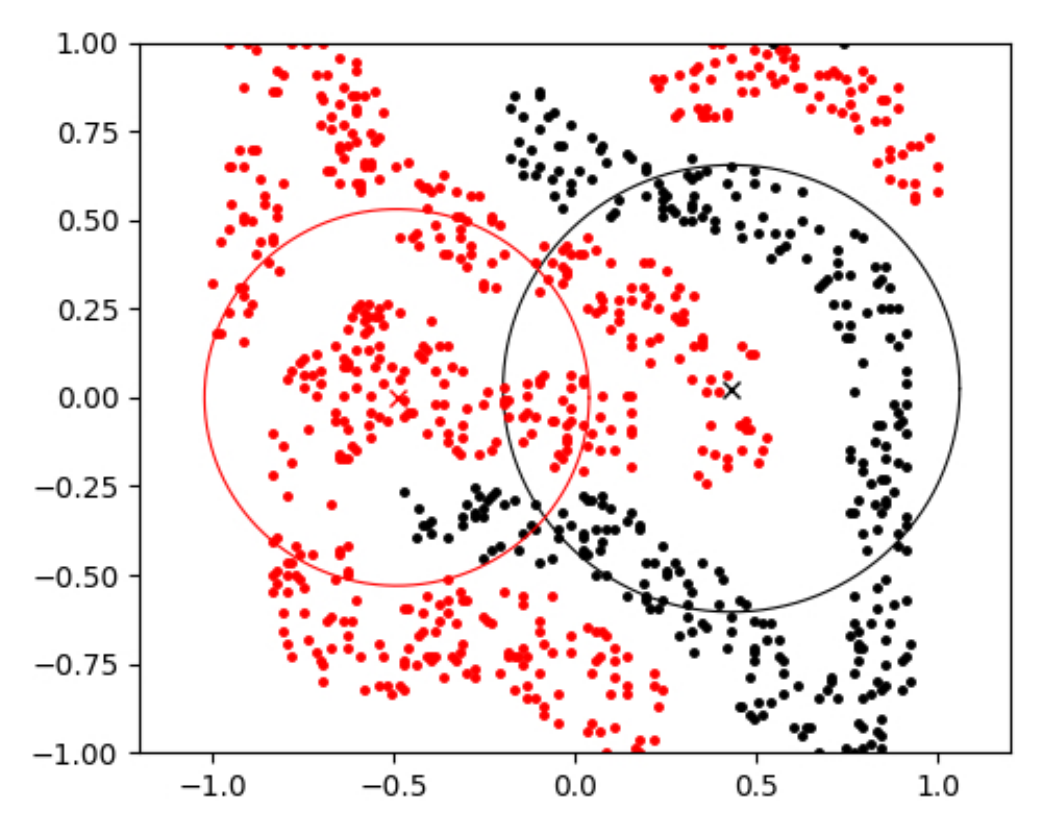}\\
	(a) \ \ \ \ \ \ \ \ \ \ \ \ \ \ \ \ \ \ \ \ \ \ \ \ \ \ \ \ (b)\\
	\includegraphics[width=1.7in]{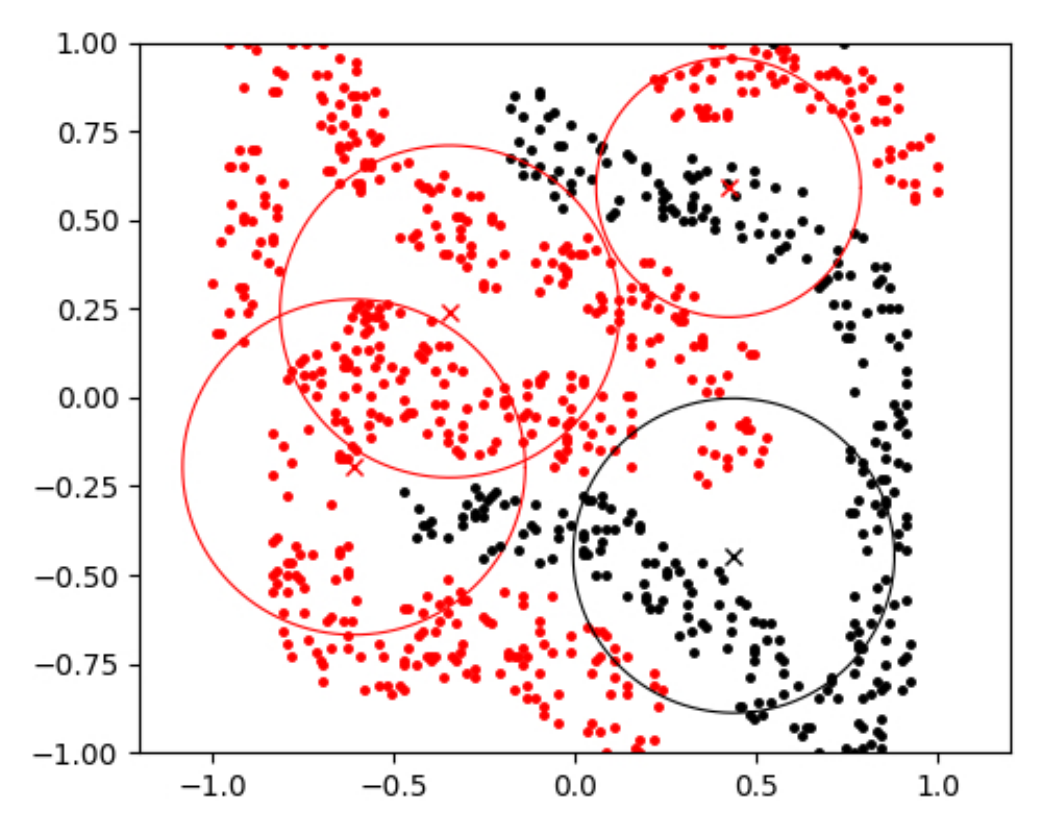}
	\includegraphics[width=1.7in]{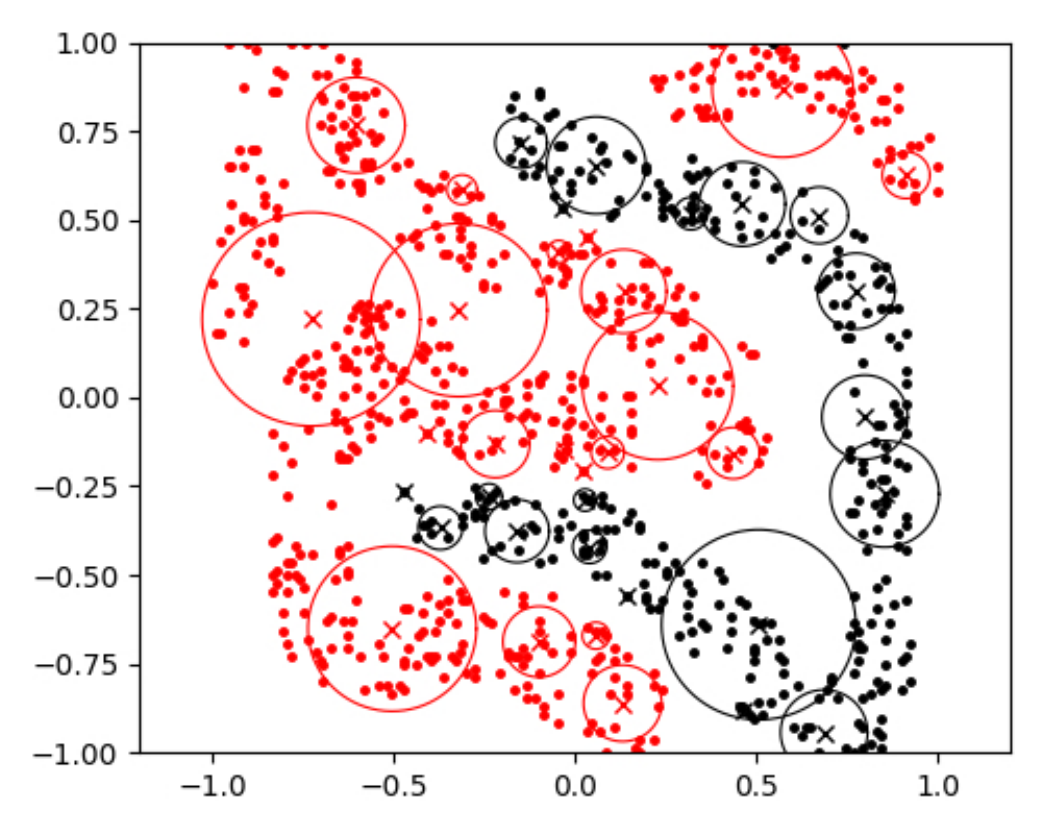}\\
	(c) \ \ \ \ \ \ \ \ \ \ \ \ \ \ \ \ \ \ \ \ \ \ \ \ \ \ \ \ (d)\\
	\includegraphics[width=1.7in]{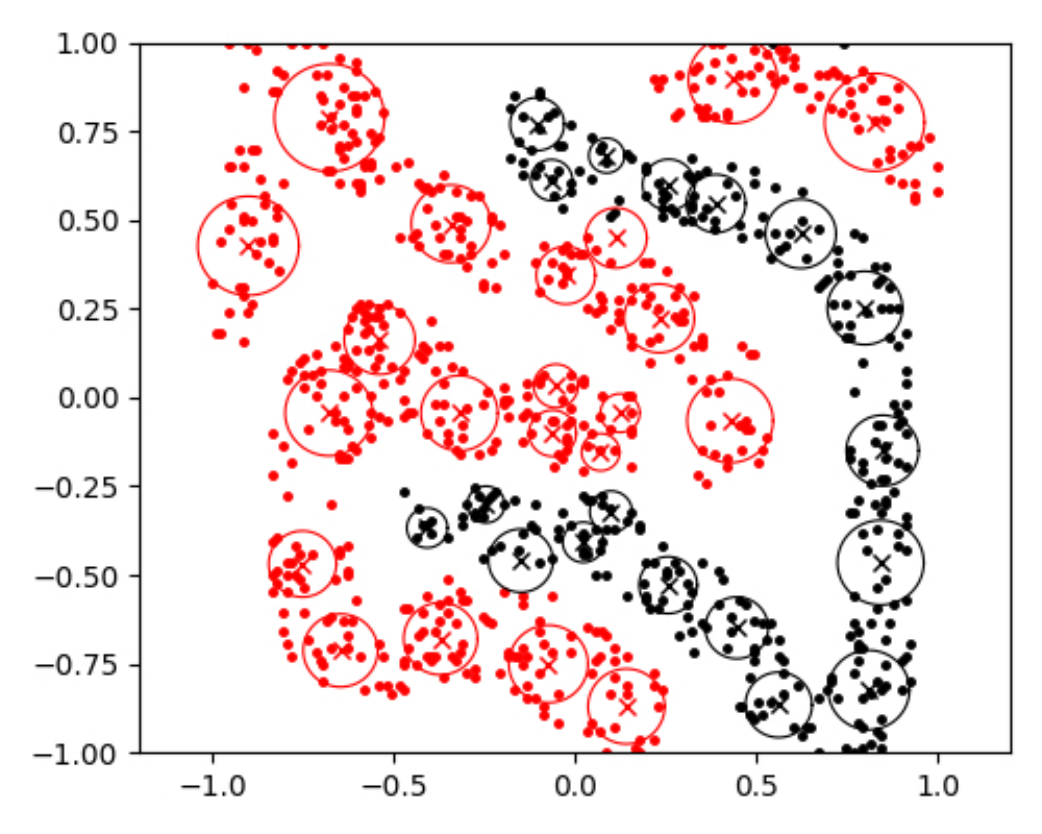}
	\includegraphics[width=1.7in]{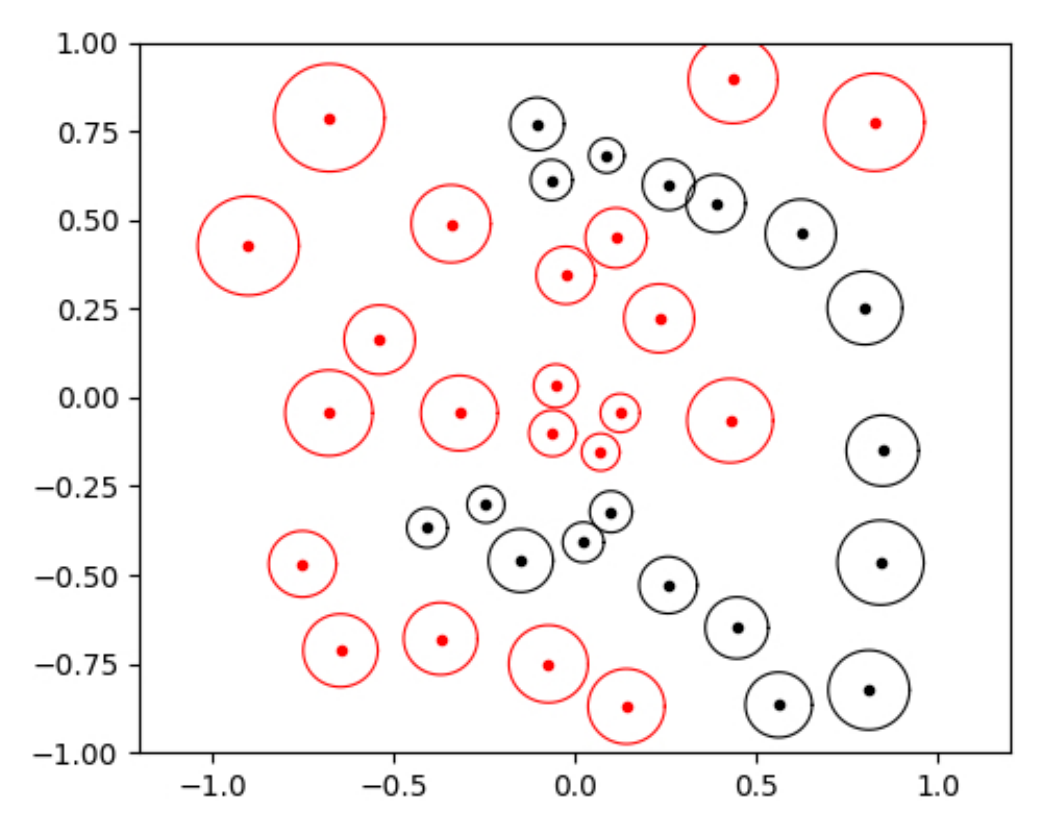}\\
	(e) \ \ \ \ \ \ \ \ \ \ \ \ \ \ \ \ \ \ \ \ \ \ \ \ \ \ \ \ (f)\\
	\caption{The granular-ball splitting generation process of the existing method
		on the data set fourclass. The colors of the two granular-balls in the figure
		(corresponding to the two-sample point colors) respectively represent the two
		types of category labels. (a) The initial granular-ball, the whole data set can
		be seen as a granular-ball to participate in subsequent iterations; (b) granular-balls generated in the first iteration; (c) granular-balls generated in the second iteration; (d) Stop splitting results; (e) Results after stopping splitting; (f)	granular-balls extracted.}
	\label{generateGB}
\end{figure}

\subsection{Granular-ball SVM}
SVM is one of the most classical and popular classification algorithms in machine learning in recent decades. It is a generalized linear classifier, which classifies data using supervised learning. The SVM classifier takes points as input and divides the data into two categories by finding a classification line (two-dimensional is a straight line, three-dimensional is a plane, and multi-dimensional is a hyperplane). If the points are taken as the input, there will be a large amount of calculation and low robustness. Considering granule balls as input, the basic model of granular ball support vector machine (GBSVM) is derived \cite{xia2022gbsvm}, and its main principle is shown in Fig. \ref{figure:GB}. Given a dataset $D = \left\{(x_i, y_i),i = 1, 2, . . ., n \right\}$, where $ y_i \in \left\{+1,-1\right\}$ denotes the label of $x_i$. In Fig. \ref{Fig:GBSVM}, the two colors represent the two types of granular-balls, which are labeled "+1" and "-1". The set of generated granular-balls is denoted by $G = \{((c_i,r_i),y_i), i = 1, 2, . . ., n\}$ as the input, where $c$ and $r$ represent the center and radius, respectively.   
 \begin{figure}[!ht]
 	\centering
 	{\includegraphics[width = 0.48\textwidth]{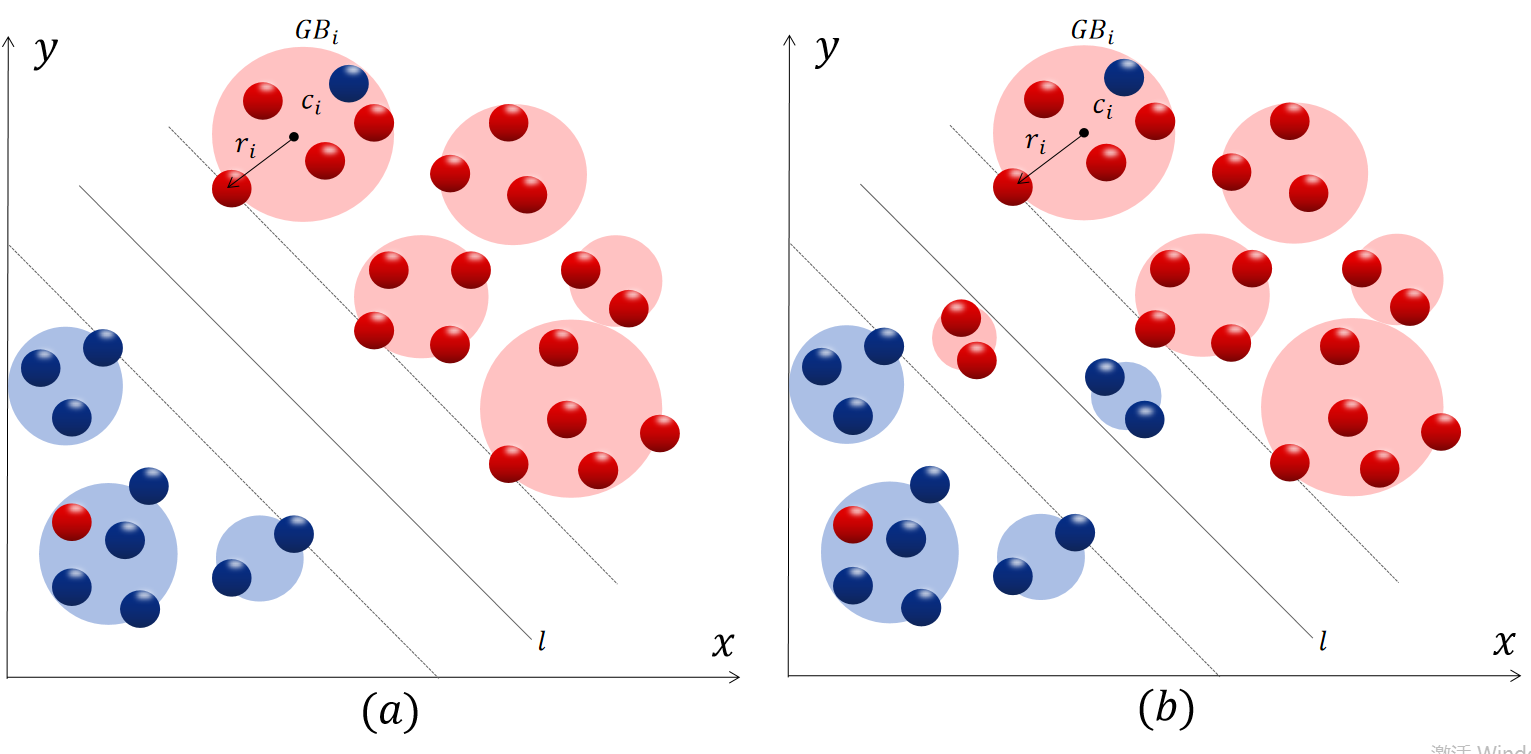}}
 	\caption{Schematic diagram of the granular-ball support vector machine. (a) The separable GBSVM; (b) The
 		inseparable GBSVM.}
 	\label{Fig:GBSVM}
 \end{figure} 

Taking the ball as input has two main advantages: Firstly, taking the ball as input significantly reduces the number of samples input, so as to improve the efficiency of training. Secondly, as shown in Fig. \ref{Fig:GBSVM}, since the overall label of a granular ball is the label with the most appearances in that granular ball, a small number of samples do not affect the label of the whole granular-ball, so the granular-ball algorithm is more robust.

By maximizing the interval and formalizing it into convex quadratic programming, the objective function of the separable GBSVM in separable classification can be obtained as follows:
\begin{equation}\label{}
\left\{\begin{aligned}
&{\rm {min}} \ \ \ \frac{1}{2}\Vert w \Vert^2, \\
&s.t. \quad y_i(wc_i+b)- \Vert w \Vert {r_i} \textgreater 1 , i=1,2,...,n, \\
\end{aligned}\right.
\end{equation}
where $b$ is the bias of the decision plane and $w$ denotes the normal vector of the decision plane.
 
In the separable GBSVM, all supported granular-balls must satisfy the constraints. However, it is necessary to introduce the slack variable $\xi$ and the penalty coefficient $C$, since some granular-balls do not satisfy the constraint in most cases. Therefore, the inseparable GBSVM model can be expressed as:
\begin{equation}\label{}
\left\{\begin{aligned}
&{\rm {min}} \ \ \ \frac{1}{2}\Vert w \Vert^2 + C \sum_{i=1}^l \xi_i, \\
&s.t. \quad y_i(wc_i+b)- \Vert w \Vert {r_i} \textgreater 1 - \xi_i , \\
&\quad\quad\ \  \xi_i \geq 0,i=1,2,...,n.
\end{aligned}\right.
\end{equation}

GBSVM can be robust to noise without using any other techniques, since the coarse granularity of a granular-ball can eliminate the effect of the minimum granularity of the label noise points as shown in Fig. \ref{Fig:GBSVM}.

\section{Granular-ball fuzzy set}\label{sec3}
\subsection{Motivation}

Most of the existing data processing methods use the finest information granularity as input, which is computationally inefficient. GBSVM combined with granular-ball computing and obtained more efficient and robust results. However, this method was not introduced to the fuzzy set. Classical FSVM was proposed by introducing the fuzzy membership of each point in the training set. The fuzzy membership $\delta_i$ is the attitude of the corresponding point $x_i$ toward one class and the parameter $\xi_i$ is a measure of error. FSVM was considered to add different membership weights to the loss terms, but it still used the most fine-grained sample points as input \cite{lin2002fuzzy}, as shown in Fig. \ref{FSVM}(a). The time cost is high for large sample data. In addition, FSVM that takes points as input is not robust to label noise. Granular-ball computing can be considered in the fuzzy classifiers, and the simple principle is shown in Fig. \ref{FSVM}(b). In this study, we propose a framework called granular-ball fuzzy set to address these issues, and the framework can be applied to all directions of fuzzy data processing.

\begin{figure}[!ht]
	\centering
	{\includegraphics[width = 0.48\textwidth]{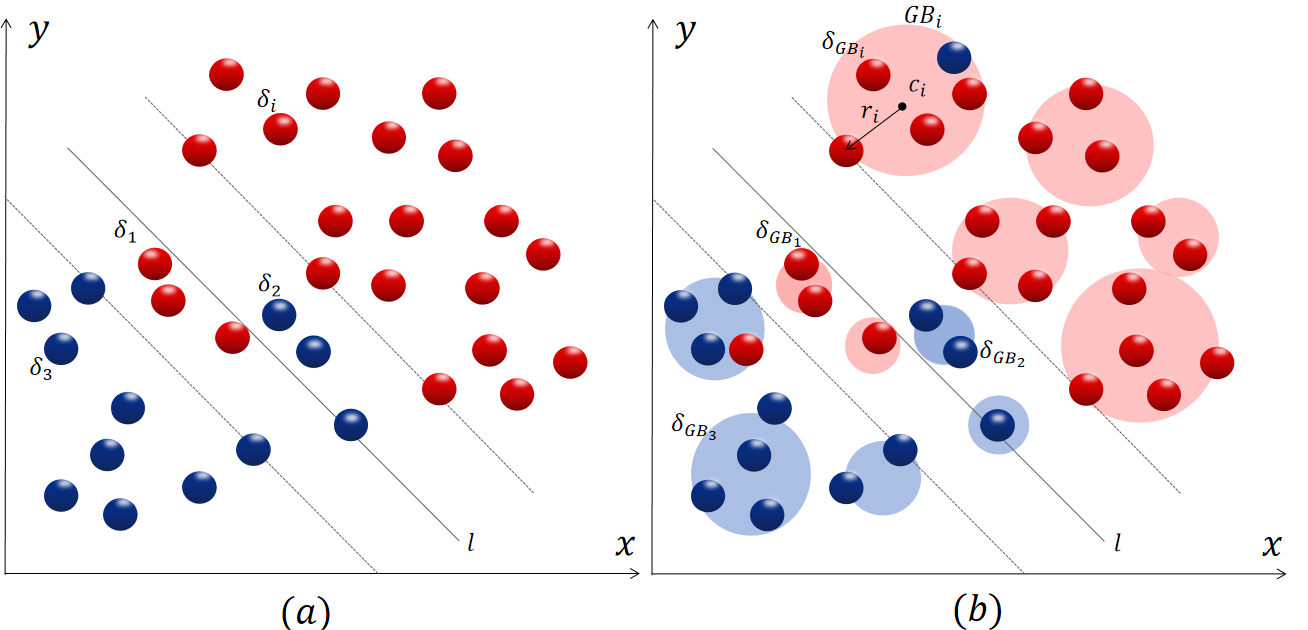}}
	\caption{Schematic diagrams of support vector machine in fuzzy set($\delta_i$ denotes the membership degree). (a) Traditional FSVM; (b) GBFSVM. }
	\label{FSVM}
\end{figure}

\subsection{Fuzzy granular-ball computing framework}

Inspired by the granular-ball computing, we believe that the traditional fuzzy data processing structure with points as input can be converted to a new structure incorporating granular-ball fuzzy set. Given a fuzzy dataset $D=\{(x_1, \delta_1),(x_2,\delta_2),...,(x_n,\delta_n)\}$, where $\delta_i$ is the degree of membership. In the new structure, different sizes of granular-balls are generated computationally, so that the dataset is reformulated as:
$$
G = \{(GB_1,\delta_{GB_1}),(GB_2,\delta_{GB_2}),...,(GB_l,\delta_{GB_l})\},
$$
where $l$ and $\delta_{GB_1}$ represent the number and membership degree of fuzzy granular-ball. Each granular-ball, $GB_i$, can be represented by its center $c_i$ and radius $r_i$. Assuming that the fuzzy data processing method is a function $f(x_i,\delta_i)$, the computing framework of fuzzy granular-ball is to transform the points into fuzzy granular-balls as input, which is expressed as:
$$
f(x_i,\delta_i) \rightarrow  f(GB_i,\delta_{GB_i}).
$$
In this computational framework, the input becomes a general description of each sub-dataset of different sizes. Taking balls as input reduces the number of input samples, thus greatly improving the efficiency of training. In addition, the size of fuzzy granular-ball can be adjusted according to the specific situation. Coarse granularity balls may not lead to lower accuracy, because larger coarse granularity balls can reduce the impact of noisy data and increase robustness to noisy. The framework called granular-ball fuzzy set can be extended to various scenarios in the field of fuzzy data processing. Therefore, the framework has three advantages: high efficiency, robustness and scalability.

\subsection{Definition of the fuzzy granular-ball}
Taking the field of fuzzy data classification as an example, most of the existing processing methods all take sample points as input, and the computation process has a high time cost and is not robust to label noise. In our previous work, we initially provided the definition of a fuzzy granular-ball with a known the membership degree of each sample point as follows \cite{shao}:

\begin{definition}\label{FGB2}
	If the membership degree ${\delta}_i$ of each sample in the dataset is known, fuzzy granular-balls are obtained by the $k$-means algorithm. For each fuzzy granular-ball, $GB_i = \left\{(x_1,{\delta}_1),(x_2,{\delta}_2),...,(x_l,{\delta}_l)\right\}$, where $l$ represents the number of samples in each granular-ball. The membership degree ${\delta}_{GB}$ of the ball is obtained as follows:
	$$
	{\delta}_{GB_i}= \frac{1}{l} \sum_{t=1}^{l} {\delta}_t.	
	$$
\end{definition}	

However, the above definition does used a general framework, its GBSVM algorithm is not designed, and its GBSVM model is incorrect, too complex and inconsistent with the FSVM. However, in practice, most of the sample points of fuzzy data sets do not contain membership degrees, so we define the fuzzy granular-ball by designing the membership function as follows:

\begin{definition}\label{FGB1}
	The set of fuzzy granular-balls, $D = \{ \left( (c_1, r_1),y_1\right), \left( (c_2, r_2),y_2\right),...,\left( (c_n, r_n),y_n\right) \}$ is generated by $k$-means algorithm ($n$ represents the number of balls). The membership degree of the fuzzy granular-ball ${GB_i}$, ${\delta}_{GB_i}$, is defined using the membership function $\mu ( x )$ as follows:
$$
{\delta}_{GB_i} = \mu(c_i).	
$$	
\end{definition}

In view of Definition \ref{FGB1} of fuzzy granular-ball, on the one hand, fuzzy granular-ball is used instead of sample points as input to improve the computational efficiency when the membership function of training samples is known. On the other hand, it is necessary to calculate the corresponding center point of each fuzzy granular-ball as the membership degree of the ball to participate in training, instead of calculating the membership degree of each sample, which significantly reduces the operation cost and improves computational efficiency. 

In addition, the overall label of the fuzzy granular-ball is defined as the label that appears most frequently in the ball. In general, a larger fuzzy granular-ball leads to greater efficiency and lower accuracy. However, since the overall label of a granular-ball is the label that appears most in the ball, the influence of noisy data in each fuzzy granular-ball can be eliminated, which makes the fuzzy granular-ball algorithm more robust.


\subsection{Fuzzy granular-ball generation method }

The generated granular-balls may also contain samples of different classes. To determine whether the fuzzy granular-ball should be further divided, we also used the concept of "purity," which is defined as the percentage of the majority sample in the fuzzy granular-ball. When the purity value of the granular-ball is too low, the further division is required to obtain a high-quality granular-ball. For example, if a granular-ball contains 20 positive samples and 80 negative samples, its purity is equal to 0.8. If the purity threshold is set to 0.9, then fuzzy granular-ball must be segmented. As is known to all, the $k$-means clustering algorithm is more suitable for generating spherical data clustering to achieve good efficiency, so the method of fuzzy granular-balls generation is designed as shown in Algorithm \ref{Alg:algoritm1}.
 
\begin{algorithm} 
	\caption{Generation of fuzzy granular-balls}
	\label{Alg:algoritm1}
	\hspace*{0.02in} {\bf Input:} 
	Fuzzy training set $D$, the purity threshold $p$;\\  
	\hspace*{0.02in} {\bf Output:} Fuzzy granular-ball $GB$;
	
	\begin{algorithmic}[1]
		\STATE The fuzzy training set $D$ is split into two balls $GB_{1}^{1}$ and $GB_{2}^{1}$ by using 2-means. The number of iteration steps is denoted as $s$ and initialized to 1. 
		\STATE  \textbf{For} each $GB_{j}^{s}$:
		\STATE  \ \ \  Calculate the center $c_{j}^{s}= \frac{1}{n_{j}^{s}} \sum_{i = 1}^{n_{j}^{s}} x_i$ of $GB_{j}^{s}$, where $x_i\in GB_{j}^{s},i=1,...,n_{j}^{s}$;
		\STATE  \ \ \  Calculate the radius $r_{j}^{s}=\frac{1}{n_{j}^{s}} \sum_{i = 1}^{n_{j}^{s}}  \left| x_i - c_{j}^{s}  \right|$;
		\STATE  \ \  The purity degrees $p_{j}^{s}$ is equal to the percentage of majority samples in $GB_{j}^{s}$;
		\STATE  \ \ \  \textbf{If} the value $p_{j}^{s}< p$  \textbf{Then} 
		\STATE  \ \ \  \ \ \ $GB_{j}^{s}$ is split by 2-means clustering algorithm.
		\STATE  \ \ \  \textbf{End if}
		\STATE  \textbf{End for}
		\STATE  \textbf{If} the purity of each granular-ball is higher than $p$ \textbf{Then}
		\STATE  \ \ \ The membership degree of fuzzy granular-ball can be obtained by $\delta_{GB_{j}^{s}}=\mu(c_{j}^{s})$
		\STATE  \textbf{Else}
		\STATE  \ \ \  $s=s+1$, return to step 2;
		\STATE  \textbf{End if}
	\end{algorithmic}
\end{algorithm}

As with granular-ball computing, the learning process of the fuzzy granular-ball classifier includes fuzzy granular-ball generation and its computational learning. Since the number of fuzzy granular-balls generated after splitting a large dataset can almost be considered as a small constant, the training time of the classification or regression process can be ignored. Therefore, the time cost of the fuzzy granular-ball classifier is mainly caused by the time of the fuzzy ball generation process. It is commendable that the fuzzy granular-ball classifier has good robustness and low time complexity \cite{G15}, which does not necessarily result in a loss of accuracy.

\section{The application of granular-ball fuzzy set}\label{sec4}
\subsection{ Granular-ball fuzzy support vector machine}
In practical problems, there are sample points that may not completely belong to a certain class and some sample points that are affected by noise and have no significance for classification. For example, when turning the steering wheel, the left turn can be 20\% or 100\%, but in both cases the corresponding label is left. However, standard support vector machine treats all training samples equally, so it is very sensitive to the noise and outlier samples mixed in another class, which reduces the generalization ability of the classifier. In view of this situation, Lin et al. \cite{lin2002fuzzy} proposed FSVM model by applying the fuzzy technology to SVM. According to the different contributions of different input samples to classification, the corresponding membership degree was assigned, so as to reduce the influence of noise and outlier samples and improve the classification performance of SVM. However, FSVM spends a lot of time solving the problem of fuzzy data classification with a large number of samples. We apply granular-ball fuzzy set to FSVM and use fuzzy balls instead of sample points as input to improve the computational efficiency. Suppose that we have a series of training points:
%
%
\begin{equation}\label{trainPoint}
\left \{ \left((c_1,r_1),\delta_1,y_1\right),\left((c_2,r_2),\delta_2,y_2\right),...,\left((c_l,r_l),\delta_l,y_l\right) \right \},
\end{equation}
where $c$, $r$ and $\delta$ represent the center, radius and membership of the granular-ball, respectively.

For the membership degree $\delta$, there are two cases in which the membership degree is known and the membership degree is unknown. When the membership degree is known, the membership degree of the fuzzy granular-balls can be obtained by Definition \ref{FGB2}, that is, the average membership degree of all sample points in each fuzzy granular-ball. For most existing data sets, membership information is not included. Therefore, it is necessary to construct membership functions to generate membership information, and the membership of fuzzy granular-ball can be obtained by Definition \ref{FGB1}.

At present, the method for constructing the membership function is primarily measured by calculating the distance from the sample to the class center. The closer the distance, the greater the degree of membership. The farther the distance is, the smaller the membership degree is.

 Denote the mean of class +1 as $x_+$ and that of class -1 as $x_-$. Let the radius of class +1 and the radius of class -1 be:
\begin{equation}\label{}
r_+ = \underset {x_i:y_i=1} {\rm {max}} \left |  x_i - x_+\right |,
\end{equation}

\begin{equation}\label{}
r_- = \underset {x_i:y_i=-1} {\rm {max}} \left | x_i-x_-\right |,
\end{equation}
Let the fuzzy membership $\delta_i$ be a function of the mean and radius of each class:
\begin{equation}\label{eq3.5}
\mu(x_i)=
\left\{\begin{aligned}
&1 - \left |   x_i- x_+\right | / (r_+ + \epsilon ), \quad if \quad y_i = 1,\\
&1 - \left |   x_i- x_-\right | / (r_- + \epsilon ), \quad if \quad y_i = -1,\\
\end{aligned}\right.
\end{equation}
where $\epsilon \textgreater 0$ is used to avoid the case $\mu(\delta_i) = 0$. There is no general rule to follow for the determination of membership function, which needs to be described according to the characteristics of different data sets.

For the fuzzy classifier problem, FSVM considers adding different membership weights to the lost items to make full use of the information of the samples. Similarly, we define the fuzzy granular-ball, and introduce the membership degree of each ball into GBSVM, and the GBFSVM model can be expressed as:

\begin{equation}\label{GBFSVM}
\left\{\begin{aligned}
&\underset {w,b,\xi_i} {\rm {min}} \ \ \ \frac{1}{2}\Vert w \Vert^2 + C \sum_{i=1}^l \delta_i\xi_i, \\
&s.t. \quad y_i(wc_i+b)- \Vert w \Vert {r_i} \geq 1 - \xi_i , i=1,2,...,n,\\
&\quad\quad\ \  \xi_i \geq 0.
\end{aligned}\right.
\end{equation}

After introducing the Lagrange multiplier $\alpha_i$ as the inequality constraint, the Lagrange function of Eq. (\ref{GBFSVM}) can be expressed as:
\begin{equation}\label{lagrangeF}
\begin{aligned}
&\mathcal{L}(w,b,\xi,\alpha,\mu,\delta) =\\
& \ \ \ \frac{1}{2} \Vert w \Vert ^2 - \sum_{i=1}^n \alpha_i (y_i(wc_i+b)- \Vert w \Vert r_i-1+\xi_i)\\
&\ \ \  +C\sum_{i=1}^n\delta_i\xi_i-\sum_{i=1}^n\mu_i\xi_i.
\end{aligned}
\end{equation}

Let $\mathcal{L}(w,b,\xi,\alpha,\mu,\delta)$ on $w$, $b$ and $\xi$ be partial derivatives equal to 0, we obtain
\begin{align}\label{pd1}
\frac{\partial{\mathcal{L}} }{\partial w} = w - \sum_{i=1}^n \alpha_i y_i c_i + \sum_{i=1}^n \alpha_i r_i \frac{w}{\Vert w \Vert}=0,\\ \label{pd2}
\frac{\partial{\mathcal{L}} }{\partial b} =  - \sum_{i=1}^n \alpha_i y_i = 0,\quad \quad \quad\quad\quad\quad\quad\quad \ \\ \label{pd3}
\frac{\partial{\mathcal{L}} }{\partial \xi} =  C - \alpha_i -\mu_i = 0. \quad \quad \quad\quad\quad\quad\quad \ 
\end{align}

Eq. (\ref{pd1}) can be rewritten as:
\begin{equation} \label{w1}
w=\frac{\left \| w\right \| \sum_{i=1}^n\alpha_i y_i c_i}{\left \| w\right \| +\sum_{i=1}^{n}\alpha_i r_i }.
\end{equation}

Square both sides of Equation (\ref{w1}) and take the square root to obtain:
\begin{equation} \label{w2}
\left \| \sum\limits_{i=1}^n\alpha_i y_i c_i  \right \|= \left | \left \| w\right \| + \sum\limits_{i=1}^{n}\alpha_i r_i \right |.
\end{equation}

Since $\left \| w\right \| > 0$, $\alpha_i >0$ and $r_i>0$, $w$ can be rewritten as:
\begin{equation} \label{w11}
	\left \| w\right \| =   \left \| \sum\limits_{i=1}^n\alpha_i y_i c_i  \right \|  - \sum\limits_{i=1}^{n}\alpha_i r_i   =    \left \| A \right \| - B,
\end{equation}	
\begin{equation} \label{w}
	w= \frac{\left( \left \| \sum\limits_{i=1}^n\alpha_i y_i c_i  \right \| - \sum\limits_{i=1}^{n}\alpha_i r_i  \right) \sum\limits_{i=1}^n\alpha_i y_i c_i }{\sum\limits_{i=1}^n\alpha_i y_i c_i}= \frac{(\left \|A \right \| - B) A }{\left \| A\right \| },
\end{equation}
where,
\begin{eqnarray}
	A=\sum_{i=1}^n \alpha_i y_i c_i \notag  , \
	B=\sum_{i=1}^n \alpha_i r_i \notag .
\end{eqnarray}

When $0< \alpha < \delta C$, $b$ can be obtained as 
\begin{equation}\label{b}
b=  \frac{1 + \left \| w\right \| r_i}{y_i} - w c_i.
\end{equation}

Putting Eqs. (\ref{pd1}), (\ref{pd2}) and (\ref{pd3}) into Eq. (\ref{lagrangeF}), the dual model of the inseparable GBFSVM is:
 
\begin{eqnarray}\label{GBFSVM2}
&& \underset {\alpha} {\rm {max}}\quad -\frac{1}{2} A^2 - \frac{1}{2} B^2 +   \Vert A \Vert  B +  \sum_{i=1}^n \alpha_i \notag, \\
&& \ s.t. \quad\quad \sum_{i=1}^n \alpha_i y_i=0 \notag ,\\
&& \quad\quad\quad\quad 0 \leq \alpha_i \leq \delta_i C, i=1,2,...,n. 
\end{eqnarray}

Eq. (\ref{GBFSVM2}) can be expressed as:

\begin{eqnarray}\label{dualGBFSVM}
&&\underset {\alpha} {\rm {max}}\quad -\frac{1}{2} \Vert w \Vert ^2 +  \sum_{i=1}^n \alpha_i \notag, \\
&& \ s.t. \quad\quad \sum_{i=1}^n \alpha_i y_i=0 \notag, \\
&& \quad\quad\quad\quad 0 \leq \alpha_i \leq \delta_i C, i=1,2,...,n, 
\end{eqnarray}
where $\Vert w \Vert$ is given by Eq. (\ref{w11}). 

The interesting phenomenon is that the obtained dual model (\ref{dualGBFSVM}) is consistent with the original model, and it is also corresponding to the FSVM and GBSVM models.
\begin{algorithm}[H]
	\caption{Granular-ball fuzzy support vector machine}
	\label{Alg:algoritm2}
	\hspace*{0.02in}{\bf Input:} 
	The fuzzy granular-ball set was $D=\left \{ \left((c_1,r_1),\delta_1,y_1\right),\left((c_2,r_2),\delta_2,y_2\right),...,\left((c_l,r_l),\delta_l,y_l\right) \right \}$. \\
	\hspace*{0.02in} {\bf Output:}
	$y$
	\begin{algorithmic}[1]
		\STATE According to Eq. (\ref{GBFSVM2}), defines the fitness function $ f=-(-\frac{1}{2} A^2 - \frac{1}{2} B^2 +  \Vert A \Vert  B +  \sum_{i=1}^n \alpha_i)$, where $\alpha_i$ is a Lagrange multiplier. That is to say, each $\alpha_i$ corresponds to a granular-ball $GB_i$, then, $\alpha=[\alpha_1,\alpha_2,\dots,\alpha_n]$.
		\STATE  Initialize the number of the particles $pop$, the maximum number of iterations $max\_iter$, the lower boundary $lb$, the upper boundary $\delta C$, the inertia factor $w$ and the learning factors $c_1$ and $c_2$;
		\STATE Use PSO to looking for the optimal $\alpha=[\alpha_1,\alpha_2,\dots,\alpha_n]$ to satisfy $\substack{\min\\\alpha}-f$;
		\STATE According to Eq. (\ref{w}), 
		calculate $\omega$;
		\STATE According to Eq. (\ref{b}), we calculated $b$.
	\end{algorithmic}
\end{algorithm}

\subsection{Granular-ball fuzzy support vector machine based on triangular fuzzy number}
\label{sec:alg}
In practice, there are various examples of fuzzy classification boundaries, particularly in medical diagnoses. The fuzzy number has a very important significance in a fuzzy system. Xue et al. \cite{shao} derived the GBSVM model based on triangular fuzzy numbers. However, GBSVM model, which is used by them has an error that the support vector is not contained in the constraint, its form is not very uniform with the traditional model. Therefore, we have corrected it in the following proposed fuzzy SVM based on triangular fuzzy numbers to classify the fuzzy data model. For a given fuzzy dataset, the fuzzy granular-ball training set can be obtained by using the granular-ball fuzzy set as follows:
\begin{equation}\label{train1}
D_1 = \left\{((c_1,r_1),{\tilde{y}}_1), ((c_2,r_2),{\tilde{y}}_2),...,((c_l,r_l),{\tilde{y}}_l) \right\},
\end{equation}
where the fuzzy number ${\tilde{y}}_j$ reflects the fuzzy category. The fuzzy classification problem involves finding a rule to infer the fuzzy number ${\tilde{y}}_j$ corresponding to any $x$ and thus reflects its fuzzy category \cite{F11}. For fuzzy information, there are generally three types of fuzzy characteristics: fuzzy positive class, i.e., the membership degree of the sample points belonging to the positive class is greater than that belonging to the negative class; fuzzy negative class, i.e., the membership degree of the sample points belonging to the negative class is greater than that belonging to the positive class; center, i.e., the membership degree of the sample points belonging to the positive class and the negative class is equal. 

For convenience, we introduce the membership of the positive class fuzzy granular-ball $\delta_+ \in [0.5, 1]$ and the membership of the negative class fuzzy granular-ball $\delta_- \in [-1, -0.5]$. Therefore, for the membership degree of the fuzzy granular-ball $\delta$, the three fuzzy features can be represented using special triangular fuzzy numbers as follows:
\begin{equation}\label{eq3}
 \begin{array}{*{20}{l}}{{{\tilde{y}}_j}={\left( {a_1, a_2, a_3} \right) }}\\{={\left\{ {\begin{array}{*{20}{l}}{\text{}{ \left( {\frac{{2 \delta \mathop{{}}\nolimits^{{2}}+ \delta -2}}{{\delta }},2 \delta -1,\frac{{2 \delta \mathop{{}}\nolimits^{{2}}-3 \delta +2}}{{ \delta }}} \right) },0.5 \le \delta \le 1},\\{\text{}{ \left({\frac{{2 \delta \mathop{{}}\nolimits^{{2}}+3 \delta +2}}{ \delta },2 \delta +1,\frac{{2 \delta \mathop{{}}\nolimits^{{2}}- \delta -2}}{ \delta }} \right) },-1 \le \delta \le -0.5.}\end{array}}\right. }}\end{array}
\end{equation}

To solve the research problem, the fuzzy granular-ball training points in the fuzzy granular-ball set $D_1$ are reordered; that is, the positive class fuzzy granular-balls are ranked in the front and the negative class fuzzy granular-balls are ranked in the back, so as to obtain the fuzzy training set in the following form:

\begin{equation}\label{train2}
\begin{aligned}
& \ \ \ D_2 = \{((c_1,r_1),{\tilde{y}}_1),...,((c_p,r_p),{\tilde{y}}_p)\\
&\ \ \  ((c_{p+1},r_{p+1}),{\tilde{y}}_{p+1}),...,((c_l,r_l),{\tilde{y}}_l)\},
\end{aligned}
\end{equation}
where $((c_t,r_t),{\tilde{y}}_t)$ is the fuzzy positive class ball $t = 1,...,p$ and $((c_i,r_i),{\tilde{y}}_i)$ is the fuzzy negative class ball $i = p+1,...,l$.

Considering the fuzzy linear separable problem corresponding to the fuzzy training set $D_2$, at this time, under the confidence level $\lambda \in [0, 1]$, the fuzzy classification problem is transformed into the following fuzzy chance constrained programming problem with $(w,b)^T$ as the decision variable:
\begin{equation}\label{eq3.0}
\left\{\begin{aligned}
&\underset {w,b} {\rm {min}} \ \ \ \frac{1}{2} \Vert w \Vert ^2 ,\\
&s.t. \quad {\rm{Pos}}\{ {\tilde{y}}_i(w \cdot c_i+b)- \Vert w \Vert {r_i} \geq 1 \} \geq \lambda,{i=1,2,...,l}.\\
\end{aligned}\right.
\end{equation}

\begin{theorem}\label{theorem3}
	Let $\tilde{y}$ in fuzzy chance constrained programming Eq. (\ref{eq3.0}) is a triangular fuzzy number, i.e. $\tilde{y}=(a_1,a_2,a_3)$. Under the confidence level $\lambda$, the clear equivalent programming of granular fuzzy chance constrained programming (\ref{eq3.0}) is the following quadratic programming:
\begin{equation}\label{eq3.3}
\left\{\begin{aligned}
&\underset {w,b} {\rm {min}} \ \ \ \frac{1}{2} \Vert w \Vert ^2 ,\\
&s.t.\quad((1-\lambda)a_{t3}+\lambda a_{t2})\cdot (wc_i+b)-\Vert w \Vert 
{r_{t}} \geq 1,\quad\quad \\
& \ \ \ \ \ \ \ \ \ \ \ \ \ \ \ \ \ \ \ \ \ \ \ \ \ \ \ \ \ \ \ \ \ \ \ \ 
\ \ \ \ \ \ \ \ \ \ \ \ \ \ \ \ 
t=1,2,...,p,\\
& \ \quad\quad((1-\lambda)a_{j1}+\lambda a_{j2})\cdot (wc_i+b)-\Vert w 
\Vert {r_j} \geq 1,\quad\quad \\
& \ \ \ \ \ \ \ \ \ \ \ \ \ \ \ \ \ \ \ \ \ \ \ \ \ \ \ \ \ \ \ \ \ \ \ \ 
\ \ \ \ \ \ \ \ \ \ \ \ \ \ \ \ j=p+1,...,l.\\
\end{aligned}\right.
\end{equation}
\end{theorem}

\begin{proof}
	From the properties of triangular fuzzy number operation, if $\tilde{y_i} = (a_{i1},a_{i2},a_{i3})$ is a triangular fuzzy number, then ${-\tilde{y}}_i(w \cdot c_i+b)+ \Vert w \Vert {r_i}+1 $ is also a triangular fuzzy number.
	
	Let  
	\begin{equation}\label{eq3.4}
	h_i^+(w,b)=
	\left\{\begin{aligned}
	&0, \ \ \quad\quad\quad\quad\quad h_i(w,b)\geq 0,\\
	&-(wc_i+b),\quad h_i(w,b)\textless 0,\\
	\end{aligned}\right.
	\end{equation}
	
	\begin{equation}\label{eq3.5}
	h_i^-(w,b)=
	\left\{\begin{aligned}
	&(wc_i+b), \ \ \quad h_i(w,b)\geq 0,\\
	&0,\quad\quad\quad\quad\quad h_i(w,b)\textless 0.\\
	\end{aligned}\right.
	\end{equation}

	Then $h_i^+(w,b)$ and $h_i^-(w,b)$ are nonnegative numbers and $h_i(w,b)=h_i^+(w,b)-h_i^-(w,b)$.
	That is,
    \begin{equation}\label{eq3.6}
    \begin{aligned}
     &{\rm{Pos}}\{ {\tilde{y}}_i(w \cdot c_i+b)- \Vert w \Vert {r_i} \geq 1} \}\geq \lambda,{i=1,2,...,l\notag\\
     & ={\rm{Pos}}\{ 1- {\tilde{y}}_i(w \cdot c_i+b) + \Vert w \Vert {r_i} \leq 0}\}\geq \lambda ,{i=1,2,...,l .
     \end{aligned}
     \end{equation}	
	
	Therefore,
	\begin{eqnarray}
	&&h_i(w,b)\cdot \tilde{y_i} +\Vert w \Vert {r_i}+1\notag\\
	&&=(h_i^+(w,b)-h_i^-(w,b))\cdot \tilde{y_i}+\Vert w \Vert {r_i}+1\notag\\
	&&=(h_i^+(w,b)-h_i^-(w,b))\cdot (a_{i1},a_{i2},a_{i3})+\Vert w \Vert {r_i}+1\notag\\
	&&=(h_i^+(w,b)a_{i1}-h_i^-(w,b)a_{i3}+\Vert w \Vert 
	{r_i}+1,\\ 
	&& \ \ \ \  h_i(w,b)a_{i2}+\Vert w \Vert {r_i} 
	+1,h_i^+(w,b)a_{i3}-\notag\\ 
	&& \ \ \ \  h_i^-(w,b)a_{i1}+\Vert w \Vert .
	{r_i}+1).\notag
	\end{eqnarray}
	
	According to Lemma \ref{de2.4}, the clear equivalent class of Eq. (\ref{eq3.6}) is obtained by 
	\begin{equation}\label{eq3.7}
	\begin{split}	
	(1-\lambda)((h_i^+(w,b)a_{i1}-h_i^-(w,b)a_{i3})+\lambda(h_i(w,b)a_{i2})\\
	+\Vert
	w \Vert {r_i}+1 \leq 0,i=1,2,...,l.
	\end{split}
	\end{equation}
	
	At the confidence level $\lambda(0 \textless \lambda \leq 1)$, Eq. (\ref{eq3.7}) can be expressed as
%
	\begin{equation}\label{eq3.8}
    \left\{\begin{aligned}
     &((1-\lambda)a_{i3}+\lambda a_{i2})\cdot (wc_i+b)-\Vert w \Vert {r_i} \geq 
  1, wc_i+b \textgreater 0,\\
     &((1-\lambda)a_{i1}+\lambda a_{i2})\cdot (wc_i+b)-\Vert w \Vert {r_i} \geq 1, wc_i+b \leq 0.\\
    \end{aligned}\right.
    \end{equation}
	
	In summary, Eq. (\ref{eq3.0}) can be expressed as:
	\begin{equation}\label{eq3.10}
\left\{\begin{aligned}
&\underset {w,b} {\rm {min}} \ \ \ \frac{1}{2} \Vert w \Vert ^2 ,\\
&s.t.\quad((1-\lambda)a_{t3}+\lambda a_{t2})\cdot (wc_t+b)-\Vert w \Vert 
{r_{t}} \geq 1,\quad\quad \\
& \ \ \ \ \ \ \ \ \ \ \ \ \ \ \ \ \ \ \ \ \ \ \ \ \ \ \ \ \ \ \ \ \ \ \ \ 
\ \ \ \ \ \ \ \ \ \  
t=1,2,...,p,\\
& \ \quad\quad((1-\lambda)a_{j1}+\lambda a_{j2})\cdot (wc_j+b)-\Vert w 
\Vert {r_j} \geq 1,\quad\quad \\
& \ \ \ \ \ \ \ \ \ \ \ \ \ \ \ \ \ \ \ \ \ \ \ \ \ \ \ \ \ \ \ \ \ \ \ \ 
\ \ \ \ \ \ \ 
j=p+1,...,l.
\end{aligned}\right.
\end{equation}
\end{proof}

By proving Theorem \ref{theorem3}, the problem (\ref{eq3.0}) is equivalent to (\ref{eq3.10}). After introducing the Lagrange multipliers $\beta_t$ and $\alpha_j$, the corresponding augmented Lagrange function of Eq. (\ref{eq3.10}) can be expressed as:
\begin{small}
\begin{align}\label{eq11}
\mathcal{L}&(w,b,\beta,\alpha) 
=\frac{1}{2} \Vert w \Vert ^2 - \sum_{t=1}^p 
\beta_t \cdot (((1-\lambda)a_{t3}+\lambda a_{t2})\cdot (wc_t+b) \notag \\
&-\Vert w \Vert {r_{t}}-1)-\sum_{j={p+1}}^l \alpha_j \cdot 
(((1-\lambda)a_{j1}+\lambda a_{j2})\cdot 
(wc_j+b)\notag \\
& - \ \Vert w \Vert {r_{j}}-1),
\end{align}  
\end{small}
where, $\beta=(\beta_1,\beta_2,...,\beta_p)^T  \in \mathbb{R}  ^p, \  \alpha=(\alpha_{p+1},...,\alpha_l)^T  \in \mathbb{R} ^{l-p}$.\\

Let $\mathcal{L}(w,b,\beta,\alpha)$ on $w$ and $b$ partial derivatives be equal to zero, and we obtain
%
\begin{align}\label{eq3.12}
\frac{\partial{\mathcal{L}} }{\partial w} &= w - M -N + \sum_{t=1}^p \beta_t 
\cdot r_t \frac{w}{\Vert w \Vert}+ \sum_{j={p+1}}^l \alpha_j \cdot r_j \frac{w}{\Vert w \Vert} =0, \quad \quad\quad \quad \quad \quad
\end{align}

\begin{align}\label{eq3.13}
\frac{\partial{\mathcal{L}} }{\partial b}= -{\sum_{t=1}^p \beta_t  ((1-\lambda)a_{t3}+\lambda a_{t2})  
}\notag \quad\quad \quad\quad \quad\quad \quad \\
  - {\sum_{j={p+1}}^l \alpha_j ((1-\lambda)a_{j1}+\lambda a_{j2})   } =0, \ \ \quad \quad \quad 
\end{align}
where,
\begin{flalign}
&M={\sum_{t=1}^p \beta_t ((1-\lambda)a_{t3}+\lambda a_{t2}) \cdot c_t 
}\notag, \\
&N={\sum_{j={p+1}}^l \alpha_j ((1-\lambda)a_{j1}+\lambda a_{j2}) \cdot c_j 
} .  
\end{flalign}
By simplifying Eq. (\ref{eq3.12}) and Eq. (\ref{eq3.13}), we can obtain:
\begin{equation}\label{eq3.14}
w= \frac{\Vert w \Vert(M+N)}{\Vert w \Vert +{\sum\limits_{t=1}^p \beta_t   r_t}+{\sum\limits_{j=p+1}^l \alpha_j r_j}}.
\end{equation}
By further simplifying Eq. (\ref{eq3.14}), we can get:

\begin{equation}\label{eq3.115}
{ \left\Vert { M + N } \right\Vert }^2 =\left( \Vert w \Vert + {\sum\limits_{t=1}^p \beta_t  r_t}+{\sum\limits_{j=p+1}^l \alpha_j r_j} \right) ^2.
\end{equation}
%
Taking the square root of Eq. (\ref{eq3.115}), since ${ \left\Vert {w} \right\Vert }, \beta, \alpha, r>0$,  ${ \left\Vert {w} \right\Vert }$ can be obtain as follows:
\begin{equation}\label{eq3.15}
\begin{array}{*{20}{l}}{{ \left\Vert {w} \right\Vert }=   {{ \left\Vert { M + N } \right\Vert }}    {-{\mathop{ \sum }\limits_{t=1}^{p}{ \beta_t r_t}}-{\mathop{ \sum }\limits_{j=p+1}^{l}{\alpha_j r_j}}} }\end{array}.
\end{equation}
Substituting Eq. (\ref{eq3.15}) into Eq. (\ref{eq3.14}), $w$ can be described as follows: 

\begin{equation}\label{membership function}
w= \frac{(\Vert E \Vert -F) \cdot E}{\Vert E \Vert},
\end{equation}\label{eq3.16}
where,
\begin{flalign}
 E=M+N; \ \ \ 
 F={\sum_{t=1}^p \beta_t r_t + \sum_{j=p+1}^l \alpha_j r_j} .
\end{flalign}

Put Eq. (\ref{eq3.13}) and (\ref{eq3.14}) into Eq. (\ref{eq11}), the dual function of Eq. (\ref{eq3.10}) is

\begin{equation}\label{eq3.17}
	\left\{\begin{aligned}
		&\underset {\alpha,\beta} {\rm {max}}\quad -\frac{1}{2} E ^2 - \frac{1}{2} F ^2 + \Vert E \Vert F + \sum_{t=1}^p \beta_t + \sum_{j=p+1}^l \alpha_j ,\\
		&s.t.\quad\sum_{t=1}^p \beta_t((1-\lambda)r_{t3}+ \lambda r_{t2})+  \\
		&\ \ \ \ \ \ \sum_{j=p+1}^l \alpha_j((1-\lambda)r_{j1}+\lambda r_{j2})=0.
	\end{aligned}\right.
\end{equation}

Eq. (\ref{eq3.17}) can be transformed into the same function as the original model of SVM:
\begin{small}
	\begin{equation}\label{eq3.18}
		\left\{\begin{aligned}
			&\underset {\alpha,\beta} {\rm {max}}\quad -\frac{1}{2} \Vert w \Vert ^2 + \sum_{t=1}^p \beta_t + \sum_{j=p+1}^l \alpha_j ,\\
			&s.t.\quad\sum_{t=1}^p \beta_t((1-\lambda)a_{t3}+\lambda a_{t2}) + \\
			&\ \ \ \ \ \ \sum_{j=p+1}^l \alpha_j((1-\lambda)a_{j1}+\lambda a_{j2})=0.\\
		\end{aligned}\right.
	\end{equation}
\end{small}

\section{Experiment}{\label{sec:experiment}}
In this section, since it is difficult to find suitable data with triangular fuzzy number characteristics, we only conducted experiments using the FSVM method for fuzzy granular-ball extension applications. The feasibility and efficiency of GBFSVM are verified by comparison with the SVM and FSVM methods. In the experiment, our hardware experimental environment is AMD Ryzen Threadripper PRO 5975WX 32-Cores 3.60 GHz, and the software experiment environment is Python 3.9. The six benchmark datasets used in the experiment are all from the public dataset UCI, and their corresponding dataset names, sample numbers and dimensions are listed in Table \ref {tab:1}. To ensure the fairness of the experiment, all models are optimized using the particle swarm optimization (PSO) algorithm with the same parameters.

\begin{table}[!ht]
	\centering
	\caption{Dataset Information}
	\setlength{\tabcolsep}{6mm}{ 
		\label{tab:1}
		\begin{tabular}{lcc}
			\hline
			Dataset 	& Samples 	& Dimensionality \\ \hline
			Fourclass 	& 862 		& 2 \\ 
			Haberman 	& 306 		& 3 \\ 
			Heart1 		& 294 		& 13 \\
			Titanic 	& 2201 		& 2 \\ 
			BreastCancer& 683 		& 9 \\ 
			Credit      & 690       & 15 \\ \hline
	\end{tabular}}
\end{table}

In this study, the data set shown in Table \ref {tab:1} was selected for experimental comparison, and the parameter $C=10$ was set. The parameters of the PSO algorithm used in the optimization process are set as follows: the dimension $dim$ denotes the number of independent variables of the objective function and $pop$ denotes the number of particles; the maximum number of iterations $max_{iter}$ is set to $1050$; the inertia coefficient $w$ is set to $0.5$; learning factor $c_1=c_2$ is equal to $1.6$. Owing to the randomization of the test seed and the fact that all the results have to meet the constraints, the triangle optimization rule is destroyed. Therefore, the experimental results included the maximum values of SVM, FSVM and GBSVM in four runs on each dataset. In the experiment, we tested each dataset with different percentages of label noise, including 0\%, 5\%, 10\%, 15\%, 20\%, 25\% and 30\%. 

\begin{table}[!ht]
	\caption{Comparsion of Accuracy between posSVM, FSVM and GBFSVM with different levels of class noise}
	\label{tab:2}
	\resizebox{\linewidth}{!}{
		\begin{tabular}{llllllllllllllllllllll}
			\hline
			\multicolumn{2}{l}{Dataset}    & Fourclass & Haberman & Heart1 & Titanic & BreastCancer & Credit \\ \hline
			\multirow{3}{*}{0\%}  
			& SVM    & 0.6705    & 0.3548   & 0.6441 & 0.3628  & 0.6350       & 0.5072 \\
			& FSVM   & 0.6532    & 0.3065   & 0.6949 &  0.6757  & 0.8686       & 0.5217 \\
			& GBFSVM & \textbf{0.7803}    & \textbf{0.7903}   & \textbf{0.7627} & \textbf{0.7778}  & \textbf{0.9927}       & \textbf{0.8261} \\
			\rule{0pt}{10pt}\multirow{3}{*}{5\%}  
			& SVM    & 0.3815    & 0.3387   & 0.5254 & 0.3220   & 0.7080       & 0.5870 \\
			& FSVM   & 0.6647    & 0.4032   & 0.5932 & 0.3628  & 0.7810       & 0.5290 \\
			& GBFSVM & \textbf{0.7514}    & \textbf{0.8226}   & \textbf{0.7458} & \textbf{0.6825}   & \textbf{0.9927  }     & \textbf{0.6739} \\
			\rule{0pt}{10pt} \multirow{3}{*}{10\%} 
			& SVM    & 0.6763    & 0.2903   & 0.5932 & 0.3492  & 0.7299       & 0.5870 \\
			& FSVM   & 0.6821    & 0.6290   & \textbf{0.6780} & 0.3741  & 0.7080       & 0.6159 \\
			& GBFSVM & \textbf{0.8266}    & \textbf{0.8065}   & 0.6441 & \textbf{0.6646} & \textbf{1.0000 }      & \textbf{0.8768} \\
			\rule{0pt}{10pt} \multirow{3}{*}{15\%} 
			& SVM    & 0.6532    & 0.3226   & 0.6271 & 0.3605  & 0.6861       & 0.6014 \\
			& FSVM   & 0.7110    & 0.3065   & 0.7288 & 0.3515  & 0.6277       & 0.4783 \\
			& GBFSVM & \textbf{0.8208}    & \textbf{0.8387 }  & \textbf{0.7457} & \textbf{0.7846} & \textbf{0.9927 }     & \textbf{0.8043} \\
			\rule{0pt}{10pt} \multirow{3}{*}{20\%} 
			& SVM    & 0.7110    & 0.3226   & 0.6271 & 0.3379  & 0.6788       & 0.4928 \\
			& FSVM   & 0.6647    & 0.3710   & \textbf{0.6949} & 0.6871  & 0.7080       & 0.5435 \\
			& GBFSVM & \textbf{0.7919}    & \textbf{0.8387}   & 0.6610 & \textbf{0.8163} & \textbf{0.9854}      & \textbf{0.8188} \\
			\rule{0pt}{10pt} \multirow{3}{*}{25\%} 
			& SVM    & 0.3815    & 0.2903   & 0.4407 & 0.3333  & 0.6788       & 0.5870 \\
			& FSVM   & 0.6416    & 0.3387   & 0.7119 & 0.3447  & 0.6423       & 0.4565 \\
			& GBFSVM & \textbf{0.7514}    & \textbf{0.7903 }  & \textbf{0.7288} & \textbf{0.7800}  & \textbf{0.9854  }     & \textbf{0.7971} \\
			\rule{0pt}{10pt} \multirow{3}{*}{30\%} 
			& SVM    & 0.3931    & 0.2903   & 0.6610 & 0.6599  & 0.6496       & 0.5290 \\
			& FSVM   & 0.6012    & 0.3065   & 0.6271 & 0.6780  & 0.7080       & 0.5942 \\
			& GBFSVM & \textbf{0.7803}    & \textbf{0.7903 }  & \textbf{0.6680} & \textbf{0.7007} & \textbf{0.9927 }      & \textbf{0.8116} \\ \hline
		\end{tabular}
	}
\end{table}

The classification accuracies of SVM, FSVM and GBFSVM under different noise levels are listed in the Table. \ref{tab:2}. As shown in Table \ref{tab:2}, GBFSVM obtains higher classification accuracies than the SVM and FSVM in most instances. In addition, the GBFSVM also consistently achieves the best results at high noise levels. This is because the label of a granular-ball is defined as the label with the most appearances in a granular-ball. Further, taking granular-balls as input can reduce the effect of label noise, which has been described in detail in Section II(B).

To further verify the high efficiency of granular-ball fuzzy set, we provide 
the running time results for SVM, FSVM and GBFSVM solved by the PSO algorithm in Table \ref {running_time}. The bold numbers indicate the best results of the three methods. Here, the time of the granular-ball purity threshold optimization was not considered, and the average running time of GBFSVM was used. In fact, the work of granular-ball calculation has been able to achieve the purity threshold adaptation \cite{xia2022efficient}. Owing to the length of this paper and the complexity of the problem be conducted in the future. It is obvious that GBSVM is much faster than the other two methods. The reason is that the application of granular-balls instead of points as input greatly reduces the number of training samples, which trains the data faster. A detailed theoretical analysis is presented in Section II. In summary, the GBFSVM is much more efficient than SVM and FSVM.

\begin{table}[!ht]
	\centering
	\caption{Comparison of the running time between SVM, FSVM and GBFSVM}
	\setlength{\tabcolsep}{4.3mm}
	\label{running_time}{
		\begin{tabular}{llll}
			\hline
			dataset      &SVM  & FSVM & GBFSVM \\ \hline
			Fourclass    & 1740.4913    & 1492.9857     & 11.0211  \\
			Haberman     & 207.2167     & 192.7256      & 79.3770  \\
			Heart1       & 206.6333     & 229.6339      & 167.05811 \\
			Titanic      & 11425.4193   & 11942.9475    & 628.1647 \\
			BreastCancer & 886.2465     & 1294.1307     & 4.7140 \\
			Credit       & 1191.9513    & 1063.8960     & 730.7979 \\ \hline
	\end{tabular}}
\end{table}

\section{Conclusions}{\label{sec6}}

This paper provides a scalable, efficient and robust algorithm framework for fuzzy big data processing by systematically defining the concept of the fuzzy granular-ball, which is available for all fuzzy data processing method. Moreover, the application of this framework reduces the time and space complexity of existing classifiers. This study extends the granular-ball fuzzy set framework to SVM classification computing and proposes the GBFSVM model and GBFSVM model based on the triangular fuzzy number. As shown in the experimental results of GBFSVM, the running time of using fuzzy granular-balls as input is much less than that of points, and its classification accuracy is higher. For datasets with different noise levels, the GBFSVM outperformed SVM in terms of robustness and effectiveness.

Despite the above advantages, there still exist some disadvantages in this paper despite of above advantages. The PSO algorithm cannot ensure a global optimal solution. Due to the limited length of this study and the complexity of the considered problem, we did not use the gradient descent method to optimize the dual model of GBFSVM. Therefore, in the future, we will research how to use the gradient descent method to solve the dual model. Moreover, GBFSVM based on the triangular fuzzy number has not been applied to solve the problem, and more reliable membership functions can be designed for more effective classification. 


\subsection{Acknowledgments}
This work was supported in part by the National Natural Science Foundation of China under Grant Nos. 62222601, 62221005 and 62176033, Key Cooperation Project of Chongqing Municipal Education Commission under Grant No. HZ2021008, and Natural Science Foundation of Chongqing under Grant No.cstc2019jcyj-cxttX0002.



\end{document}